\documentclass[journal,onecolumn,12pt]{IEEEtran}
\usepackage[utf8]{inputenc}
\usepackage{amsmath,amssymb}
\usepackage{graphicx}                                   
\usepackage{url}   
\usepackage{cite} 
\usepackage[tight,footnotesize]{subfigure}  
\usepackage{amsthm}
\usepackage{array,multirow}
\usepackage{psfrag}
\usepackage{color, colortbl}  
\definecolor{LightCyan}{rgb}{0.88,1,1}
\definecolor{Gray}{rgb}{0.82,0.82,0.82}

\newcommand{\xbf}{\mathbf{x}}

\newcommand{\wbf}{\mathbf{w}}

\newcommand{\gammabar}{\overline{\gamma}}

\newcommand{\etilde}{\widetilde{e}}

\newcommand{\mubar}{\overline{\mu}}

\newcommand{\wbftilde}{\widetilde{\wbf}}

\theoremstyle{plain}
\newtheorem{thm}{Theorem}

\newtheorem{cor}{Corollary}

\ifCLASSINFOpdf
\else
\fi
\begin{document}
%
%
%

\title{Robustness Analysis of the Data-Selective \\ Volterra NLMS Algorithm}
%
%
%

\author{Javad Sharafi and Abbas Maarefparvar
	\thanks{(J. Sharafi) and (A. Maarefparvar) Imam Ali University, Tehran, Iran.}
	\thanks{E-mail addresses: javadsharafi@grad.kashanu.ac.ir, a.marefparvar@gmail.com.}
}

%
%
%

\vspace{1cm}

%
%

\maketitle

\begin{abstract}
Recently, the data-selective adaptive Volterra filters have been proposed; however, up to now, there are not any theoretical analyses on its behavior rather than numerical simulations.
Therefore, in this paper, we analyze the robustness (in the sense of $l_2$-stability) of the data-selective Volterra normalized least-mean-square (DS-VNLMS) algorithm.
First, we study the local robustness of this algorithm at any iteration, then we propose a global bound for the error/discrepancy in the coefficient vector.
Also, we demonstrate that the DS-VNLMS algorithm improves the parameter estimation for the majority of the iterations that an update is implemented.
Moreover, we also prove that if the noise bound is known, then we can set the DS-VNLMS so that it never degrades the estimate.
The simulation results corroborate the validity of the executed analysis and demonstrate that the DS-VNLMS algorithm is robust against noise, no matter how its parameters are adopted.
\end{abstract}



%

\section{Introduction}
\label{sec:Introduction}

Nonlinear systems have been utilized in many real-wold problems, such as nonlinear echo cancellation~\cite{Azpicueta_Volterra_echo_cancellation_taslp2011}, nonlinear controllers~\cite{He_Vol_noise_controller_aut2016}, wireless sensor networks~\cite{Prado_volterra_wsn_IJCNN2018}, biological systems~\cite{Berger_biological_proc2010}, audio processing~\cite{George_audio_nonlinear_SP2013}, to mention but a few.
Indeed, when an online nonlinear solution is required, the adaptive Volterra filter (AVF) is the most appealing candidate~\cite{Claser_rlsdcd_volterra_eusipco2016,Tan_ac_noise_volterra_ciea2009}.
However, the fundamental drawback of this technique is its high computational resources since the number of AVF coefficients increases exponentially with the filter order and geometrically with the filter memory.

Recently, a data-selective AVF (DS-AVF) has been proposed to reduce the computational cost of the AVF algorithms~\cite{Silva_DS_VAF_cssp2018}.
The data-selection strategy, apart from maintaining the advantages of the conventional AVF algorithm, assesses the incoming data before utilizing them in the learning process.
This approach can improve the accuracy, the robustness against noise, and the computational complexity of the learning process by preventing the algorithm from updating adaptive coefficients when there is not enough innovation in input data~\cite{Hamed_SMrobustness_jasp2017,Zhang_sm_nlms_error_bound_tcasii2014,Hamed_robust_smnlms_sam2016,Diniz_improved_partial_icassp2016}.

To the best of our knowledge, since the DS-AVF had been proposed in 2018, no study on the properties of this algorithm is executed, whereas there are many works in literature to address the theoretical behavior of the conventional AVFs\cite{Mumolo_stability_rvf_spl1999,Sayadi_mse_avf_tsp1999,Chao_error_surface_volterra_mwscas2004,Motonaka_sma_volterra_apsipa2018}.
Therefore, in this work, similarly to the study for linear filters~\cite{Hamed_SMrobustness_jasp2017,Hamed_robust_smnlms_sam2016}, we analyze the robustness (in the sense of $l_2$-stability) of the data-selective Volterra normalized least-mean-square (DS-VNLMS) algorithm.
To this end, first, we introduce the robustness criterion, then we propose the local and the global robustness properties of the DS-VNLMS algorithm.
Moreover, we study the situations where the noise bound is assumed known and unknown.

This paper is organized as follows. Sections~\ref{sec:Volterra} and~\ref{sec:robustness_criterion} provide a brief review of Volterra filters and the robustness criterion for AVFs, respectively.
In Section~\ref{sec:robustness_DS_VNLMS}, the local and the global robustness of the DS-VNLMS algorithm is studied.
Also, the cases of known and unknown noise bound are discussed in this section.
The validity of our analysis is verified in Section~\ref{sec:simulations}.
Finally, the conclusions are drawn in Section~\ref{sec:conclusions}.

{\it Notation:} Scalars are denoted by lowercase letters.
Vectors (matrices) are presented by lowercase (uppercase) boldface letters.
The $l_2$-norm of a vector $\wbf\in\mathbb{R}^{N}$ is defined by $\|\wbf\|^2=\sum_{i=0}^N|w_i|^2$.
Moreover, the superscript $(\cdot)^T$ denotes the vector or matrix transpose operator, and $\mathbb{R}_+$ stands for the positive real numbers.


\section{A short review on Volterra series} \label{sec:Volterra}

Suppose that $x(k)$ and $d(k)$ are the input and the desired signals of a system at the time instant $k$, and they are related to each other by a nonlinear, time-invariant, finite-memory, causal, continuous relationship $d(k)=f(\xbf_1(k))$, where $\xbf_1(k)=[x(k)~x(k-1)~\cdots~x(k-N)]^T$ and $N$ is the system memory length.
A truncated Volterra series expansion of order P can be used to evaluate the signal $d(k)$ as
\begin{align}
d(k)=\sum_{p=0}^PW_p(\xbf_1(k))+n(k), \label{eq:desired-signal}
\end{align}
where $n(k)$ stands for the measurement noise. Also, by employing the triangular structure of the Volterra series~\cite{Mathews_polynomial_signal_processing_book2000,Hamed_Volterra_eusipco2019}, $W_p(\xbf_1(k))$ can be defined by
\begin{align}
W_p(\xbf_1(k))\triangleq\sum_{l_1=0}^N\ldots\hspace{-1mm}\sum_{l_p=l_{p-1}}^N \hspace{-1mm}w_p(l_1,\ldots,l_p)\prod_{i=1}^p x(n-l_i),
\label{eq:cinque}
\end{align}
where $w_p(l_1,\ldots,l_p)$ is the Volterra kernel of order $p$, for all $l_1,\ldots,l_p$.
Without loose of generality, we can consider the constant term of the Volterra
series expansion, $w_0$, identical to zero; thus, the target of the Volterra adaptive filter is to compute the Volterra kernels for all $l_1,\ldots,l_p$ and $p=1,\ldots,P$.

In order to convenience, Equations~\eqref{eq:desired-signal} and~\eqref{eq:cinque} can be represented in a compact formula~\cite{Mathews_polynomial_signal_processing_book2000,Diniz_adaptiveFiltering_book2013}.
Indeed, denote by $\xbf_p(k)$ and $\wbf_p$ the vector generated by all input sample products appearing in~\eqref{eq:cinque} and the vector containing the corresponding Volterra kernels $w_p(l_1,\ldots,l_p)$, respectively.
Then, by defining $W_p(\xbf_1(k))=\wbf_p^T\xbf_p(k)$, Equation~\eqref{eq:desired-signal} can be reformulated as
\begin{align}
d(k)=\wbf^T\xbf(k)+n(k), \label{eq:desired_vectorial}
\end{align}
where $\xbf(k)\triangleq[\xbf_1^T(k)~\cdots~\xbf_P^T(k)]^T$, $\wbf\triangleq[\wbf_1^T~\cdots~\wbf_P^T]^T$.


\section{Robustness criterion} \label{sec:robustness_criterion}

Suppose that the unknown system is given by $\wbf_*$, using Equation~\eqref{eq:desired_vectorial}, the desired signal can be expressed by $d(k)=y_*(k)+n(k)$, where $y_*(k)=\wbf_*^T\xbf(k)$.
Moreover, assume that the sequence of the noise signal $\{n(k)\}$ has finite energy~\cite{Sayed_adaptiveFilters_book2008}, i.e.,
\begin{align}
\sum_{k=0}^j|n(k)|^2<\infty,\qquad\text{for~all~}j. \label{eq:noise_condition}
\end{align}
Our target is to estimate $y_*(k)$.
To this end, suppose that $\hat{y}_{k|k}$ is an approximation of $y_*(k)$, where it is only dependent on $d(j)$ for $j=0,\cdots,k$.
For a particular $\eta\in\mathbb{R}_+$, we intend to compute the approximations $\hat{y}_{k|k}\in\{\hat{y}_{0|0},\hat{y}_{1|1},\cdots,\hat{y}_{M|M}\}$, so that for any $\wbf_*$ and any noise signal satisfying~\eqref{eq:noise_condition}, the criterion below is guaranteed~\cite{Hamed_SMrobustness_jasp2017}:
\begin{align}
\frac{\sum\limits_{k=0}^j \|\hat{y}_{k|k}-y_*(k)\|^2}{\wbftilde^T(0)\wbftilde(0)+\sum_{k=0}^j|n(k)|^2}<\eta^2,\quad
{\rm for~all~} j=0,\cdots,M, \label{eq:criterion}
\end{align}
where $\wbftilde(0)$ is the difference between the unknown system and the initial guess; i.e., $\wbftilde(0)=\wbf_*-\wbf(0)$.

To interpret~\eqref{eq:criterion}, it should be noted that the numerator counts the estimation-error energy up to iteration $j$, and the denominator incorporated the energy of noise up to instant time $j$ and the energy of the error caused by the initial guess.
Therefore, this criterion stipulates us to find the approximations $\{\hat{y}_{k|k}\}$ so that the ratio of the estimation-error energy to the energy of unreliability does not become greater than $\eta^2$.
In other words, when~\eqref{eq:criterion} is true, bounded disturbance energies imply bounded estimation-error energies; thus, the algorithm resulting $\{\hat{y}_{k|k}\}$ is robust.


\section{Robustness of the DS-VNLMS algorithm} \label{sec:robustness_DS_VNLMS}

In this section, first, we review the recursion rule of the DS-VNLMS algorithm, then we analyze the local robustness of the algorithm in the sense of $l_2$-stability.
Finally, the global robustness property for the DS-VNLMS algorithm is presented.

The data selection in the DS-VNLMS algorithm is performed by means of the set-membership filtering approach~\cite{Diniz_adaptiveFiltering_book2013,Hamed_trinion_tcasii2017,Hamed_improved_ietsp2019}.
Indeed, the DS-VNLMS algorithm updates the Volterra kernels when the magnitude of the output estimation error is greater than a predetermined positive value $\gammabar\in\mathbb{R}_+$, and its update equation is characterized by~\cite{Silva_DS_VAF_cssp2018}
\begin{align}
\wbf(k+1)=\wbf(k)+\frac{\mu(k)}{\xbf^T(k)\xbf(k)+\delta}e(k)\xbf(k), \label{eq:updare_equation}
\end{align} 
where $e(k)$ is the error signal, and it is described by $e(k)=d(k)-\wbf^T(k)\xbf(k)$.
Moreover, $\mu(k)$ is the step-size parameters and is defined by
\begin{align}
\mu(k)=\left\{\begin{array}{ll}1-\frac{\gammabar}{|e(k)|}&\text{if~}|e(k)|>\gammabar,\\0&\text{otherwise}.\end{array}\right. \label{eq:step-size}
\end{align}
Generally, $\gammabar$ is chosen based on some {\it a priori} knowledge of the problem, such as the information about the measurement noise~\cite{Diniz_adaptiveFiltering_book2013,Diniz_DS_AF_tsp2018}.
Also, $\delta$ is a small positive constant, and it is utilized to prevent division by zero when $\xbf^T(k)\xbf(k)=0$.

To facilitate the analysis of the DS-VNLMS algorithm, the recursion rule~\eqref{eq:updare_equation} can be rewritten as
\begin{align}
\wbf(k+1)=\wbf(k)+\frac{\mubar(k)}{\alpha(k)}e(k)\xbf(k)f(e(k),\gammabar), \label{eq:revised_update_equation}
\end{align} 
where $\mubar(k)=1-\frac{\gammabar}{|e(k)|}$, $\alpha(k)=\xbf^T(k)\xbf(k)+\delta$, and $f:\mathbb{R}\times\mathbb{R}_+\rightarrow\{0,1\}$ is the indicator function defined by
\begin{align}
f(e(k),\gammabar)=\left\{\begin{array}{ll}1&\text{if~}|e(k)|>\gammabar,\\0&\text{otherwise}.\end{array}\right. \label{eq:function_f}
\end{align}
Also, assume a system identification scenario where the unknown system is denoted by $\wbf_*$ and the desired signal is obtained by $d(k)=\wbf_*^T\xbf(k)+n(k)$.
Moreover, denote by $\wbftilde(k)$ the discrepancy between the unknown system $\wbf_*$ and the adaptive Volterra kernels $\wbf(k)$, i.e., $\wbftilde(k)=\wbf_*-\wbf(k)$.
Hence, the error signal can be expressed by
\begin{align}
e(k)=d(k)-\wbf^T(k)\xbf(k)=\underbrace{\wbftilde^T(k)\xbf(k)}_{\etilde(k)}+n(k), \label{eq:error_signal}
\end{align}
where $\etilde(k)=\wbftilde^T(k)\xbf(k)$ stands for the noiseless error.

\begin{thm}[Local robustness of the DS-VNLMS] \label{thm:local_robustness}
	For the DS-VNLMS algorithm, we always have
	\begin{align}
	\|\wbftilde(k+1)\|^2 = \|\wbftilde(k)\|^2, \text{ if } f(e(k),\gammabar) = 0     \label{eq:local_robustness_f0}
	\end{align}
	or
	\begin{align}
	\|\wbftilde(k+1)\|^2+\frac{\mubar(k)}{\alpha(k)}\etilde^2(k) 
	< \|\wbftilde(k)\|^2+\frac{\mubar(k)}{\alpha(k)} n^2(k), \label{eq:local_robustness_f1}
	\end{align}
	if $f(e(k),\gammabar) = 1$.
\end{thm}

\begin{proof}
	From $\wbf_*$, subtract both sides of~\eqref{eq:revised_update_equation} and use the definition of $\wbftilde(k)$, we attain
	\begin{align}
	\wbftilde(k+1)=\wbftilde-\frac{\mubar}{\alpha}e\xbf f,
	\end{align}
	where the time index $k$ and the arguments of function $f$ are eliminated to simplify the mathematical notations.
	After decomposing $e(k)$ utilizing Equation~\eqref{eq:error_signal}, we get
	\begin{align}
	\wbftilde(k+1) = \wbftilde - \frac{\mubar}{\alpha} \etilde \xbf f - \frac{\mubar}{\alpha} n \xbf f. \label{eq:robust_aux01}
	\end{align}
	By multiplying each side of the above equation to its transpose and performing some mathematical manipulations, we acquire
	\begin{align}
	\| \wbftilde(k+1) \|^2 
	=&\|\wbftilde \|^2-2\frac{\mubar}{\alpha}\etilde^2f-2\frac{\mubar}{\alpha}n\etilde f +(\etilde+n)^2\frac{\mubar^2}{\alpha^2}\|\xbf\|^2f^2    \nonumber \\
	=&\|\wbftilde\|^2+(\etilde+n)^2\frac{\mubar^2}{\alpha^2}\|\xbf\|^2f^2+\frac{\mubar}{\alpha}n^2 f  -(\etilde+n)^2\frac{\mubar}{\alpha}f-\frac{\mubar}{\alpha}\etilde^2f.   \label{eq:robustness_derivation_1}                                 
	\end{align}
	If we rearrange the above equation, we get
	\begin{align}
	\| \wbftilde(k+1) \|^2 + \frac{\mubar f}{\alpha} \etilde^2 
	= \| \wbftilde \|^2 + \frac{\mubar f}{\alpha} n^2 + c_1 c_2  ,   \label{eq:energy_relation}
	\end{align}
	in which 
	\begin{align}
	c_1 =\frac{\mubar f}{\alpha} (\etilde + n)^2,\qquad
	c_2 =\frac{\mubar f}{\alpha} \| \xbf \|^2 - 1    .
	\end{align}
	On the one hand, for $f=0$ in~\eqref{eq:energy_relation}, we obtain
	\begin{align}
	\| \wbftilde(k+1) \|^2 = \| \wbftilde(k) \|^2,
	\end{align}
	and it proves the first statement of the theorem; i.e., Equation~\eqref{eq:local_robustness_f0}.
	On the other hand, for $f=1$, we have $|e|>\gammabar>0$ by~\eqref{eq:function_f}; thus $0<\mubar<1$.
	They result in $c_1>0$ since $\alpha=\|\xbf\|^2+\delta$ is positive.
	Furthermore, $\frac{\|\xbf\|^2}{\alpha}<1$ since $\delta>0$; thus, $0<\frac{\mubar f}{\alpha}\|\xbf\|^2<1$.
	It leads to $c_2<0$, and we obtain $c_1c_2<0$ when $f=1$.
	Therefore, after eliminating $c_1c_2<0$ from~\eqref{eq:energy_relation}, we get
	\begin{align}
	\| \wbftilde(k+1) \|^2 + \frac{\mubar}{\alpha} \etilde^2 
	< \| \wbftilde \|^2 + \frac{\mubar}{\alpha} n^2,  
	\end{align}
	and it proves the second statement of the theorem; i.e., Equation~\eqref{eq:local_robustness_f1}.
\end{proof}

Theorem~\ref{thm:local_robustness} describes the coefficient deviation of the DS-VNLMS algorithm from any iteration $k$ to $k+1$.
To be more clear, Equation~\eqref{eq:local_robustness_f0} shows that the coefficient deviation does not change, when no update is implemented; however, Equation~\eqref{eq:local_robustness_f1} states that $\|\wbftilde(k+1)\|^2$ is bounded by a linear combination of $\|\wbftilde(k)\|^2$, $\etilde^2(k)$, and $n^2(k)$, when an update is performed.

We can now provide the global robustness property of the DS-VNLMS algorithm in Corollary~\ref{cor:global_robustness}.

\begin{cor}[Global robustness of the DS-VNLMS] \label{cor:global_robustness}
	Suppose that the DS-VNLMS algorithm is executed from $k=0$ (initialization) to an iteration $K$. We always have
	\begin{align}
	\dfrac{\| \wbftilde(K) \|^2+\sum\limits_{k \in {\cal K}_{\rm up}}\frac{\mubar(k)}{\alpha(k)}\etilde^2(k)}{\| \wbftilde(0) \|^2+\sum\limits_{k \in {\cal K}_{\rm up}}\frac{\mubar(k)}{\alpha(k)}n^2(k)}< 1, \label{eq:cor_result}
	\end{align}
	where ${\cal K}_{\rm up}$ stands for the set of the iterations that $\wbf(k)$ is updated.
\end{cor}

\begin{proof}
	Assume that ${\cal K}=\{ 0,1,2,\ldots,K-1\}$.
	Also, let us denote by ${\cal K}_{\rm up}^c={\cal K}\setminus{\cal K}_{\rm up}$ the set of iteration indexes that the Volterra kernels are not updated.
	By Theorem~\ref{thm:local_robustness}, \eqref{eq:local_robustness_f1} is satisfied for all $k\in{\cal K}_{\rm up}$.
	Thus, summing up this inequality for all $k\in{\cal K}_{\rm up}$ yields
	\begin{align}
	\sum_{k \in {\cal K}_{\rm up}} \Big(\| \wbftilde(k+1) \|^2 + \frac{\mubar(k)}{\alpha(k)} \etilde^2(k) \Big) < \sum_{k \in {\cal K}_{\rm up}} \Big(\| \wbftilde(k) \|^2 + \frac{\mubar(k)}{\alpha(k)} n^2(k)\Big). \label{eq:robustness_accumulation_f1}
	\end{align}
	Also, using~\eqref{eq:local_robustness_f0}, for all $k \in {\cal K}_{\rm up}^c$, we have
	\begin{align}
	\sum_{k \in {\cal K}_{\rm up}^c} \| \wbf(k+1) \|^2 = \sum_{k \in {\cal K}_{\rm up}^c} \| \wbf(k) \|^2. \label{eq:robustness_accumulation_f0}
	\end{align}
	By integrating~\eqref{eq:robustness_accumulation_f1} and~\eqref{eq:robustness_accumulation_f0}, we get
	\begin{align}
	\sum_{k \in {\cal K}} \| \wbftilde(k+1) \|^2 
	+ \sum_{k \in {\cal K}_{\rm up}} \frac{\mubar(k)}{\alpha(k)} \etilde^2(k)  
	< \sum_{k \in {\cal K}} \| \wbftilde(k) \|^2 
	+ \sum_{k \in {\cal K}_{\rm up}} \frac{\mubar(k)}{\alpha(k)} n^2(k). \label{eq:robustness_accumulation} 
	\end{align}
	Note that we can eliminate various $\|\wbftilde(k)\|^2$ from both sides of the above inequality and acquire
	\begin{align}
	\| \wbftilde(K) \|^2 + \sum_{k \in {\cal K}_{\rm up}}\frac{\mubar(k)}{\alpha(k)}\etilde^2(k)
	< \| \wbftilde(0) \|^2 + \sum_{k \in {\cal K}_{\rm up}}\frac{\mubar(k)}{\alpha(k)}n^2(k).
	\end{align} 
	Assuming that the right-hand side of the above inequality is nonzero, we get
	\begin{align}
	\dfrac{\| \wbftilde(K) \|^2 + \sum\limits_{k \in {\cal K}_{\rm up}}\frac{\mubar(k)}{\alpha(k)}\etilde^2(k)}{\| \wbftilde(0) \|^2 + 
		\sum\limits_{k \in {\cal K}_{\rm up}}\frac{\mubar(k)}{\alpha(k)}n^2(k)} <1,
	\end{align}
	and it terminates the proof.
\end{proof}

Note that, by Corollary~\ref{cor:global_robustness}, the $l_2$-stability of the DS-VNLMS algorithm from its uncertainties $\{ \wbftilde(0)$, $\{ n(k) \}_{0\leq k\leq K} \}$ to its errors $\{ \wbftilde(K), \{ \etilde(k) \}_{0\leq k\leq K} \}$ is assured independent of the selection of $\gammabar$; however, the $l_2$-stability of the conventional VNLMS algorithm is dependent on the selection of the step-size parameter, and it should be adopted small enough to ensure the $l_2$-stability.


\subsection{Convergence of $\{\|\wbftilde(k)\|^2\}$ with unknown  noise bound}\label{sub:ds-vnlms-unbounded-noise}

The demonstrated results in the previous section give us some bounds for the evolution of  
$\{\|\wbftilde(k)\|^2\}$ in terms of other parameters.
whereas, in practice, we have observed that the DS-VNLMS algorithm shows a well-behaved convergence for the 
sequence $\{\|\wbftilde(k)\|^2\}$, that is for majority of iterations we get $\|\wbftilde(k+1)\|^2 \leq \|\wbftilde(k)\|^2$.
Thus, in this section, we examine in which situations the sequence $\{\|\wbftilde(k)\|^2\}$ 
is (and is not) decreasing.

\begin{cor}\label{cor:ds_vnlms_decreasing}
	When an update happens (i.e., $f(e(k),\gammabar) = 1$), $\etilde^2(k) \geq n^2(k)$ results in $\| \wbftilde(k+1) \|^2  < \| \wbftilde(k) \|^2$.
\end{cor}
\begin{proof}
	If we rearrange the terms in~\eqref{eq:local_robustness_f1}, we have
	\begin{align}
	\| \wbftilde(k+1) \|^2 + \frac{\mubar(k)}{\alpha(k)} \left( \etilde^2(k) - n^2(k) \right) 
	< \| \wbftilde(k) \|^2,     
	\end{align}
	which is valid for $f(e(k),\gammabar) = 1$.
	Note that when $f(e(k),\gammabar) = 1$, we have $\alpha(k) \in \mathbb{R}_+$ and $\mubar(k) \in (0,1)$; thus, $\frac{\mubar(k)}{\alpha(k)} > 0$. 
	Hence, when $f(e(k),\gammabar) = 1$ and $\etilde^2(k) \geq n^2(k)$, we get $\frac{\mubar(k)}{\alpha(k)} \left( \etilde^2(k) - n^2(k) \right)\geq0$. 
	As the result, when an update happens, we obtain $\etilde^2(k) \geq n^2(k)  \Rightarrow  \| \wbftilde(k+1) \|^2  < \| \wbftilde(k) \|^2$. 
\end{proof}

It is good to mention that Corollary~\ref{cor:ds_vnlms_decreasing} affirms when an update is occurred by the DS-VNLMS algorithm and the energy of the error signal $e^2(k)$ is dominated by $\etilde^2(k)$, then the improvement in the estimate of $\wbf(k+1)$ is guaranteed.

For the first iterations of the DS-VNLMS algorithm, the absolute value of the error signal is large; as a results, we have $|e(k)|>\gammabar$ and $\etilde^2(k)>n^2(k)$.
It leads to the monotonic decreasing sequence $\{\|\wbftilde(k)\|^2\}$ in the transient period.
Also, when there is not enough innovation in the input signal during the transient period, the DS-VNLMS algorithm does not execute any updates, and we get $\| \wbftilde(k+1) \|^2  = \| \wbftilde(k) \|^2$.
Therefore, for any $k$ in the transient period, we have $\| \wbftilde(k+1) \|^2  \leq \| \wbftilde(k) \|^2$ with very high probability.
However, after the convergence, for few iterations of the DS-VNLMS algorithm we get $\| \wbftilde(k+1) \|^2  > \| \wbftilde(k) \|^2$.
Indeed, counting the exact number of iterations satisfying this inequality is not possible, but we can compute an upper bound probability for the occurrence of this event by
\begin{align}
\mathbb{P}[\|\wbftilde(k+1)\|^2 > \|\wbftilde(k)\|^2]    &\leq \mathbb{P}[\{|e(k)|>\gammabar\}\cap\{\etilde^2(k)<n^2(k)\}]      \nonumber\\
&<\mathbb{P}[|e(k)|>\gammabar]={\rm erfc}\left(\sqrt{\frac{\tau}{2}}\right) ,  \label{eq:ds_vnlms_probability}
\end{align}
where ${\rm erfc}(\cdot)$ is the complementary error 
function~\cite{Proakis_DigitalCommunications_book1995}. 
The details of the last equality can be observed in~\cite{Galdino_errorbound_ISCAS2006} by defining $\gammabar=\sqrt{\tau\sigma_n^2}$, where $\tau \in \mathbb{R}_+$ (as a rule of thumb $1\leq\tau\leq5$) and by simulating the error signal $e(k)$ as a zero-mean Gaussian random variable with variance $\sigma_n^2$.

Note that the probability of attaining $\|\wbftilde(k+1)\|^2 > \|\wbftilde(k)\|^2$ is insignificant. 
As example, for $\tau=3, 4$, and 5 we have ${\rm erfc}\Big(\sqrt{\frac{\tau}{2}}\Big)=0.0832, 	
0.0455$, and $0.0253$, respectively.
Therefore, for the majority of the iterations of the DS-VNLMS algorithm, we have $\| \wbftilde(k+1) \|^2  \leq \| \wbftilde(k) \|^2$; in other words, that the DS-VNLMS algorithm takes advantages of the input data efficiently.
In contrast with the classical adaptive Volterra algorithms, we demonstrated that for the DS-VNLMS algorithm rarely we have $\| \wbftilde(k+1) \|^2  > \| \wbftilde(k) \|^2$.
Also, we will verify this property experimentally in Section~\ref{sec:simulations}.


\subsection{Convergence of $\{\|\wbftilde(k)\|^2\}$ with known noise bound}\label{sub:ds-vnlms-bounded-noise}

Here, we show that when we know the noise bound, we can adopt the threshold parameter $\gammabar$ of the DS-VNLMS algorithm in such a way that $\{\|\wbftilde(k)\|^2\}$ becomes a monotonic decreasing sequence.

\begin{thm}[Strong Local Robustness of DS-VNLMS]\label{thm:strong-local-robustness-ds-vnlms}
	Suppose that the noise is bounded by the constant $C \in \mathbb{R}_+$, namely, $|n(k)|\leq C$, for all $k\in\mathbb{N}$. 
	If we select $\gammabar \geq 2C$, then $\{\|\wbftilde(k)\|^2\}$ is a monotonic decreasing sequence; that is $\|\wbftilde(k+1)\|^2\leq\|\wbftilde(k)\|^2,\forall k$. 
\end{thm}
\begin{proof}
	For $f(e(k),\gammabar)=1$ we have $|e(k)| = |\etilde(k) + n(k)|>\gammabar$.
	This implies that (i) for $\etilde(k)\geq0$ we get $\etilde(k) > \gammabar - n(k)$ or (ii) for $\etilde(k)\leq0$ we have $\etilde(k) < -\gammabar - n(k)$. 
	Note that $n(k) \in [-C,C]$ and $\gammabar \in [2C,\infty)$ by the hypothesis of the theorem; thus, the bound for $\etilde(k)$ by searching the minimum of (i) and the maximum of (ii) can be obtained as follows:\\
	(i) $\etilde(k) > \gammabar - n(k) \Rightarrow  \etilde_{\rm min} > \gammabar - C \geq C$; \\
	(ii) $\etilde(k) <-\gammabar - n(k) \Rightarrow  \etilde_{\rm max} <-\gammabar + C \leq -C$. \\
	By using (i) and (ii), we can conclude that if $\gammabar \geq 2C$, then $| \etilde(k) | > C$.
	This means that $| \etilde(k) | > | n(k) |$, for all $k\in\mathbb{N}$.
	Subsequently, Corollary~\ref{cor:ds_vnlms_decreasing} results in $\|\wbftilde(k+1)\|^2 < \|\wbftilde(k)\|^2$, for all $k\in\mathbb{N}$, when $f(e(k),\gammabar)=1$. 
	Moreover, for $f(e(k),\gammabar)=0$, we get $\|\wbftilde(k+1)\|^2 = \|\wbftilde(k)\|^2$. 
	Hence, for $\gammabar \geq 2C \Rightarrow \|\wbftilde(k+1)\|^2\leq\|\wbftilde(k)\|^2$, for all $k\in\mathbb{N}$.  
\end{proof}

\begin{cor}[Strong Global Robustness of DS-VNLMS]\label{cor:strong-global-robustness-ds-vnlms}
	Suppose that the DS-VNLMS algorithm is running from iteration $0$ to a given iteration $K\in\mathbb{N}$. 
	If $\gammabar \geq 2C$, then we have $\|\wbftilde(K)\|^2 \leq \|\wbftilde(0)\|^2$, where the equality is guaranteed only when no update is implemented during all the iterations. 
\end{cor}

The proof of this Corollary is not presented since, by using the same procedure of Corollary~\ref{cor:global_robustness}, it is a trivial consequence of Theorem~\ref{thm:strong-local-robustness-ds-vnlms}.


\subsection{Time-varying $\gammabar(k)$} \label{sub:ds-vnlms-time-varying-gammabar}  

Note that $\gammabar$ shows a trade-off between the convergence rate and the computational efficiency of the DS-VNLMS algorithm. 
Indeed, a large value for $\gammabar$ leads to low computational complexity by reducing the update rate, whereas it can decrease the convergence speed due to avoid update in the transient period.
Hence, selecting a suitable $\gammabar$ is fundamental in taking advantage of data selection approach.
As an alternative technique, we can adopt a time-varying error bound $\gammabar(k)$ as 
$\gammabar(k) \triangleq \sqrt{\tau(k) \sigma_n^2}$~\cite{Hamed_SMrobustness_jasp2017}, where 
\begin{align}\label{eq:gammabar-timevar}
\tau(k) \triangleq\begin{cases}
\text{Low value (e.g., $\tau(k) \in [1,5]$)} \\ \hspace{2cm}\text{if $k \in$ transient period, }   \\
\text{High value (e.g., $\tau(k) \in [5,9]$)} \\ \hspace{2cm}\text{if $k \in $ steady-state.}
\end{cases}
\end{align}
This $\gammabar$ results in a great compromise between the computational resources and the performance of the DS-VNLMS algorithm.
Moreover, when the noise bound $C$ is known, in the steady-state period, $\gammabar$ should be adopted less than or equal to $C$.

To use the $\gammabar(k)$ proposed in~\eqref{eq:gammabar-timevar}, the algorithm has to be capable of monitoring the environment to recognize a transition from transient to steady-state periods.
To this end, we can monitor $|e(k)|$. 
In other words, we can create a window with the length $E \in \mathbb{N}$ including Boolean variables denoting the iterations where an update is implemented in the $E$ recent iterations.
Then, if we detect many updates in the window, we assume that the algorithm is in the transient period; otherwise, we are in the steady-state period.


\section{Simulations} \label{sec:simulations}

In this section, we verify the robustness of the DS-VNLMS algorithm in system identification scenarios.
The nonlinear Volterra channel is given by
\begin{align}
d(k)=&-0.76x(k)+0.5x^2(k)+2x(k)x(k-2)\nonumber\\
&-0.5x^2(k-3)+n(k),
\end{align}
where $n(k)$ is a zero-mean white Gaussian noise (WGN) with the variance $\sigma_n^2=0.01$.
The order and the memory-length of the adaptive Volterra filter are equal to 3.
The robustness has been verified considering two different input signals: (i) a zero-mean WGN with the unit variance, (ii) a first-order autoregressive (AR(1)) signal produced by $x(k)=0.95x(k-1)+m(k)$, where $m(k)$ is a zero-mean WGN with the unit variance.
The regularization parameter is adopted as $\delta=10^{-9}$, and the algorithms are initialized with the null vector.
Also, we have tested the robustness using two different values of $\gammabar$; i.e., $\sqrt{5\sigma_n^2}$ and $\sqrt{2\sigma_n^2}$.
Moreover, let us denote the right-hand side (RHS) and the left-hand side (LHS) of~\eqref{eq:local_robustness_f1} by $r(k)$ and $l(k)$, respectively.

\begin{figure}[t!]
	\centering
	\subfigure[b][]{\includegraphics[width=.48\linewidth]{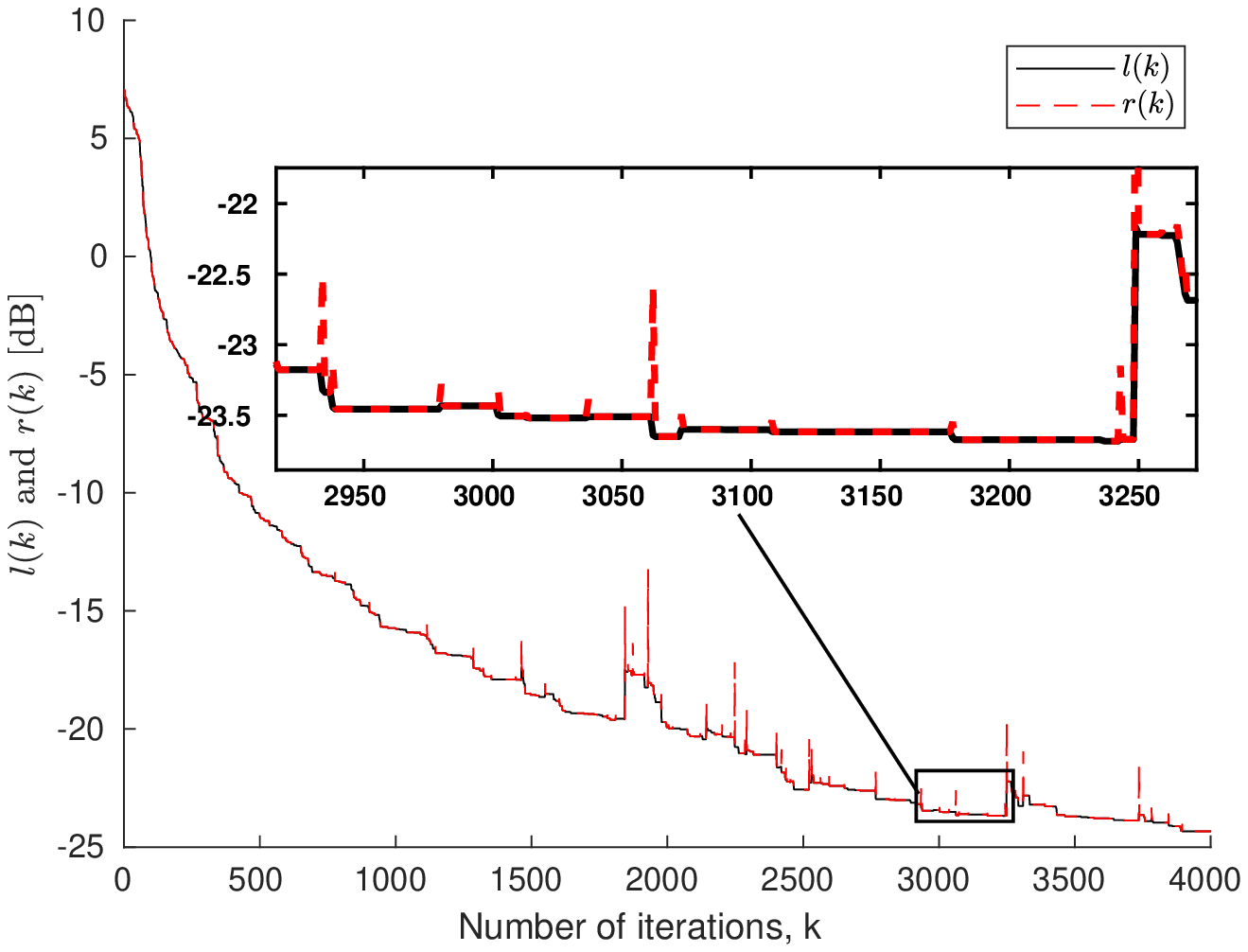}
		\label{fig:gamma5_wgn}}
	\subfigure[b][]{\includegraphics[width=.48\linewidth]{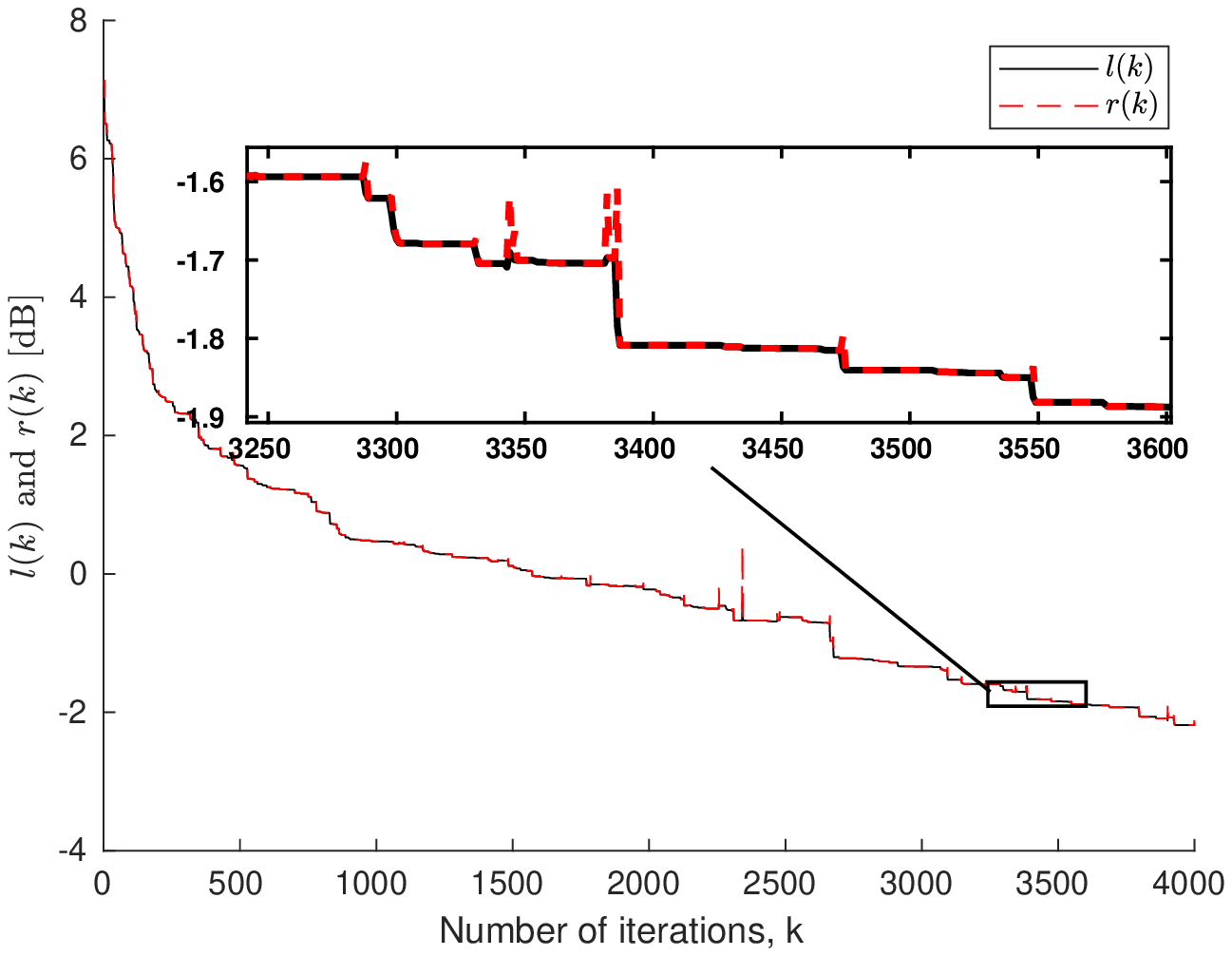}
		\label{fig:gamma5_ar1}}
	\caption{The values of $l(k)$ and $r(k)$ over the iterations when $\gammabar=\sqrt{5\sigma_n^2}$: (a) the WGN input signal; (b) the AR(1) input signal. \label{fig:large_gammabar}}
\end{figure}

Figures~\ref{fig:gamma5_wgn} and~\ref{fig:gamma5_ar1} show $l(k)$ and $r(k)$ for WGN input signal and AR(1) input signal, respectively, when $\gammabar=\sqrt{5\sigma_n^2}$.
Also, Figures~\ref{fig:gamma2_wgn} and~\ref{fig:gamma2_ar1} illustrate the same results but assuming $\gammabar=\sqrt{2\sigma_n^2}$.
In all figures, we can observe that, for both input signals and both values of $\gammabar$, $l(k)$ is strictly below $r(k)$ or is overlapped with $r(k)$; that is, $l(k)\leq r(k)$.
Note that when $l(k)$ and $r(k)$ overlap each other, it means that $f(e(k),\gammabar)=0$ and the DS-VNLMS does not update the adaptive Volterra kernels and $l(k)=r(k)$, otherwise $l(k)<r(k)$ is true and an update is performed.
Thus, these figures substantiate Theorem~\ref{thm:local_robustness}.

\begin{figure}[t!]
	\centering
	\subfigure[b][]{\includegraphics[width=.48\linewidth]{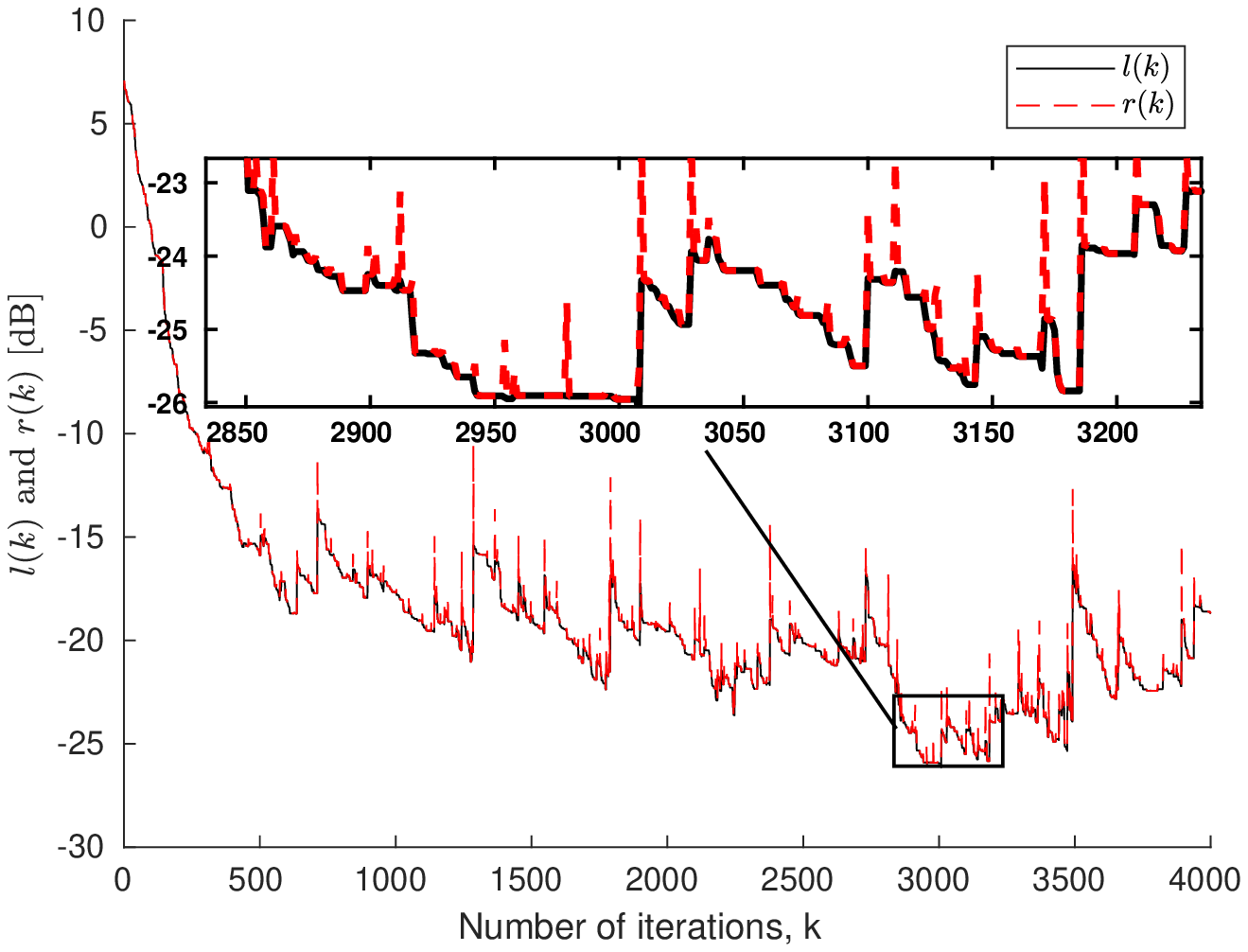}
		\label{fig:gamma2_wgn}}
	\subfigure[b][]{\includegraphics[width=.48\linewidth]{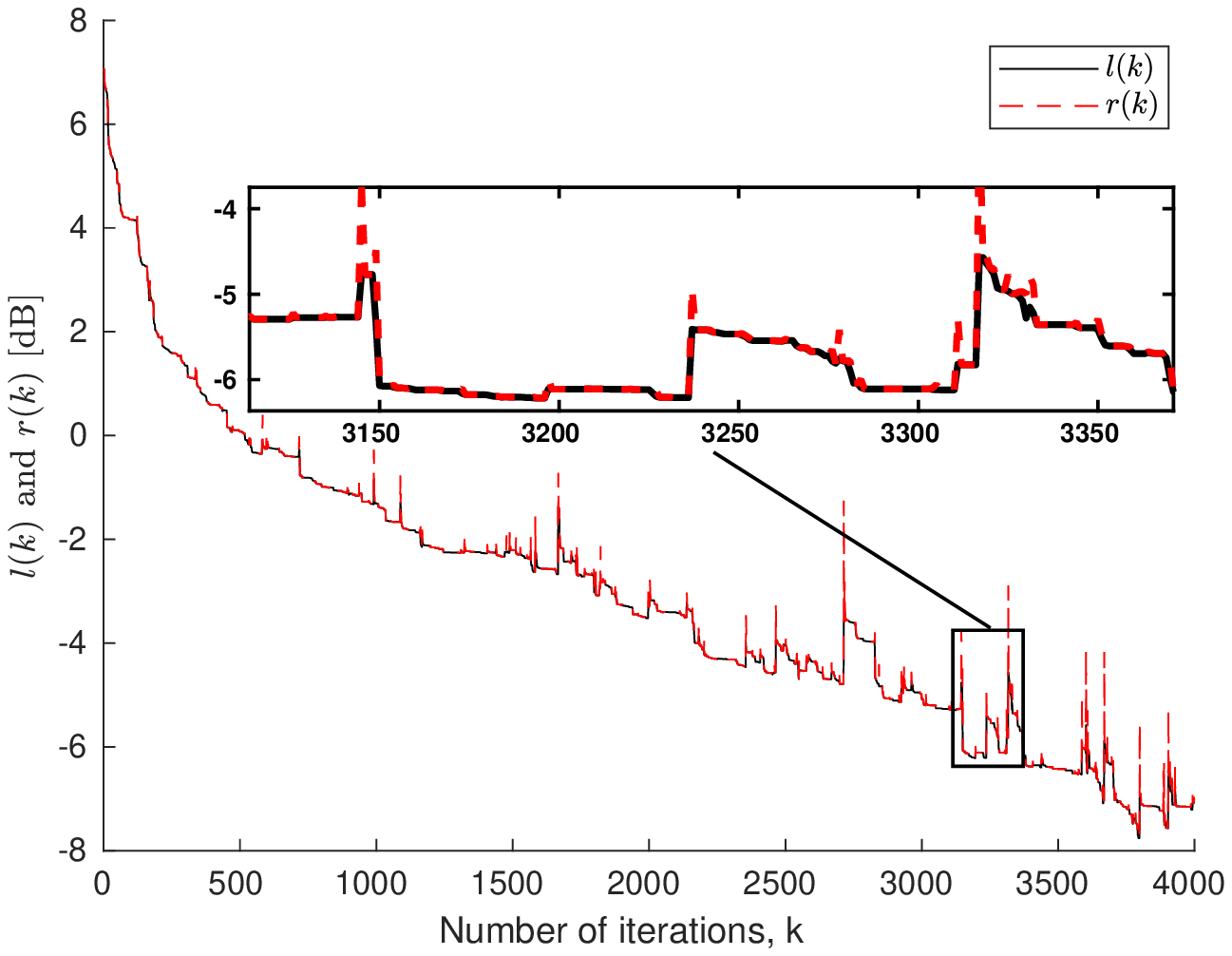}
		\label{fig:gamma2_ar1}}
	\caption{The values of $l(k)$ and $r(k)$ over the iterations when $\gammabar=\sqrt{2\sigma_n^2}$: (a) the WGN input signal; (b) the AR(1) input signal. \label{fig:small_gammabar}}
\end{figure}

When the VNLMS and the DS-VNLMS algorithms have been implemented, Figures~\ref{fig:Volterra_all} and~\ref{fig:Volterra_all_correlated} illustrate the sequence  $\{\|\wbftilde(k)\|^2\}$ for WGN input signal and AR(1) input signal, respectively.
Three different cases for the DS-VNLMS algorithm have been considered: fixed $\gammabar$ with unknown noise bound (magenta solid line), fixed $\gammabar$ with known noise bound $C=0.1$ (blue solid line), and time-varying $\gammabar(k)$, described by $\sqrt{5\sigma_n^2}$ for the transient period and $\sqrt{9\sigma_n^2}$ for the steady-state period, 
with unknown noise bound (black solid line). 
When we utilized the time-varying $\gammabar(k)$, the window length is adopted as $E=20$, and we supposed that the algorithm is in the steady-state when the number of updates in the window 
is less than 5. 
Moreover, when we have executed the VNLMS algorithm, two different step-sizes are chosen as $\mu=0.8$ and $\mu=0.3$. The larger step-size leads to fast convergence and high misadjustment, whereas $\mu=0.3$ results in slow convergence and low misadjustment.

In Figure~\ref{fig:Volterra_all}, we observe that the sequence $\{\|\wbftilde(k)\|^2\}$ presented by the magenta curve increases only $25$ times during the $2500$ iterations, it means that the DS-VNLMS algorithm did not promote the adaptive Volterra kernels only in $25$ iterations.
In the case of the correlated input signal, this number if 17 among 2500 iterations.
Therefore, in this experiment, for the WGN and the AR(1) input signals, we get $\mathbb{P}[\|\wbftilde(k+1)\|^2>\|\wbftilde(k)\|^2] = 0.01$ and $0.0068$, where they are lower than the upper bound given by ${\rm erfc}(\sqrt{2.5})=0.0253$, as described in Subsection~\ref{sub:ds-vnlms-unbounded-noise}.
Moreover, note that the inequality $\|\wbftilde(k+1)\|^2>\|\wbftilde(k)\|^2$ did not happen in the transient period since 
in this period $\etilde^2(k)$ is generally large due to a remarkable mismatch between $\wbf(k)$ and $\wbf_o$.
It means that the condition described in Corollary~\ref{cor:ds_vnlms_decreasing} is regularly held.

Furthermore, by verifying the blue curves in Figures~\ref{fig:Volterra_all} and~\ref{fig:Volterra_all_correlated}, we can observe that when the noise bound is known, by adopting $\gammabar \geq 2B$, the sequence $\{\|\wbftilde(k)\|^2\}$ is monotonic decreasing; This substantiate Theorem~\ref{thm:strong-local-robustness-ds-vnlms} and Corollary~\ref{cor:strong-global-robustness-ds-vnlms}.
In Figure~\ref{fig:Volterra_all}, the sequence $\{\|\wbftilde(k)\|^2\}$ denoted by the black curve increases only one time, and in Figure~\ref{fig:Volterra_all_correlated}, the black curve never increases.
This corroborate the benefit of applying a time-varying $\gammabar(k)$ when the noise bound is unknown.
Note that the behavior of $\{\|\wbftilde(k)\|^2\}$ for the VNLMS algorithm is extremely irregular, and in many iterations we have $\|\wbftilde(k+1)\|^2>\|\wbftilde(k)\|^2$.
This happens since the VNLMS algorithm performed many unnecessary updates.

Therefore, the DS-VNLMS algorithm as compared to the VNLMS algorithm has high convergence rate, low computational burden, and well-behaved sequence $\{\|\wbftilde(k)\|^2\}$.
In Figure~\ref{fig:Volterra_all}, the update rates of the DS-VNLMS algorithm in the magenta, blue, and black curves are 5$\%$, 1.4$\%$, and 1.7$\%$, respectively.
Also, in Figure~\ref{fig:Volterra_all}, the update rates for the magenta, blue, and black curves are 4.9$\%$, 1.1$\%$, and 1.4$\%$, respectively.

\begin{figure}[t!]
	\centering
	\subfigure[b][]{\includegraphics[width=.48\linewidth]{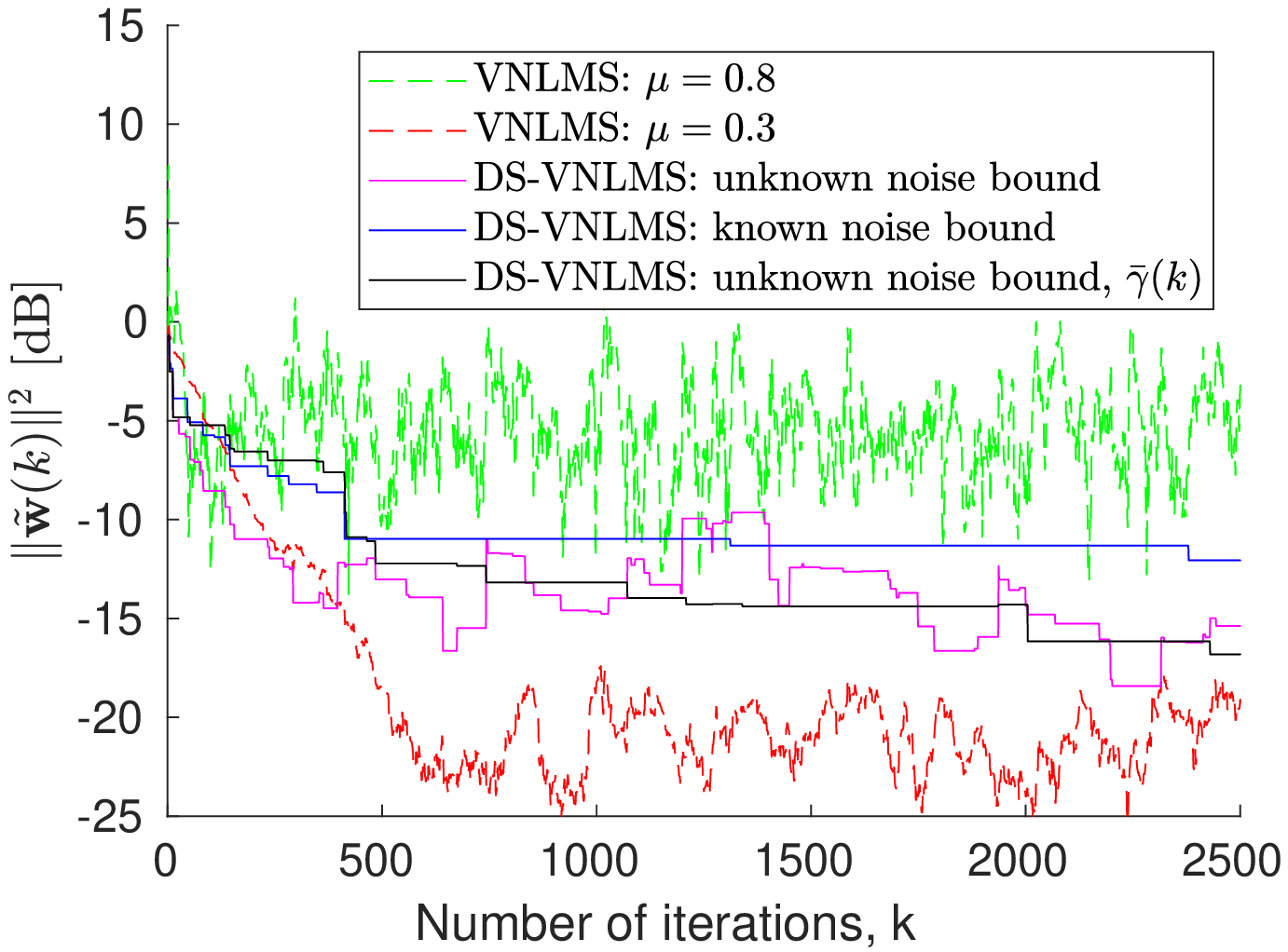}
		\label{fig:Volterra_all}}
	\subfigure[b][]{\includegraphics[width=.48\linewidth]{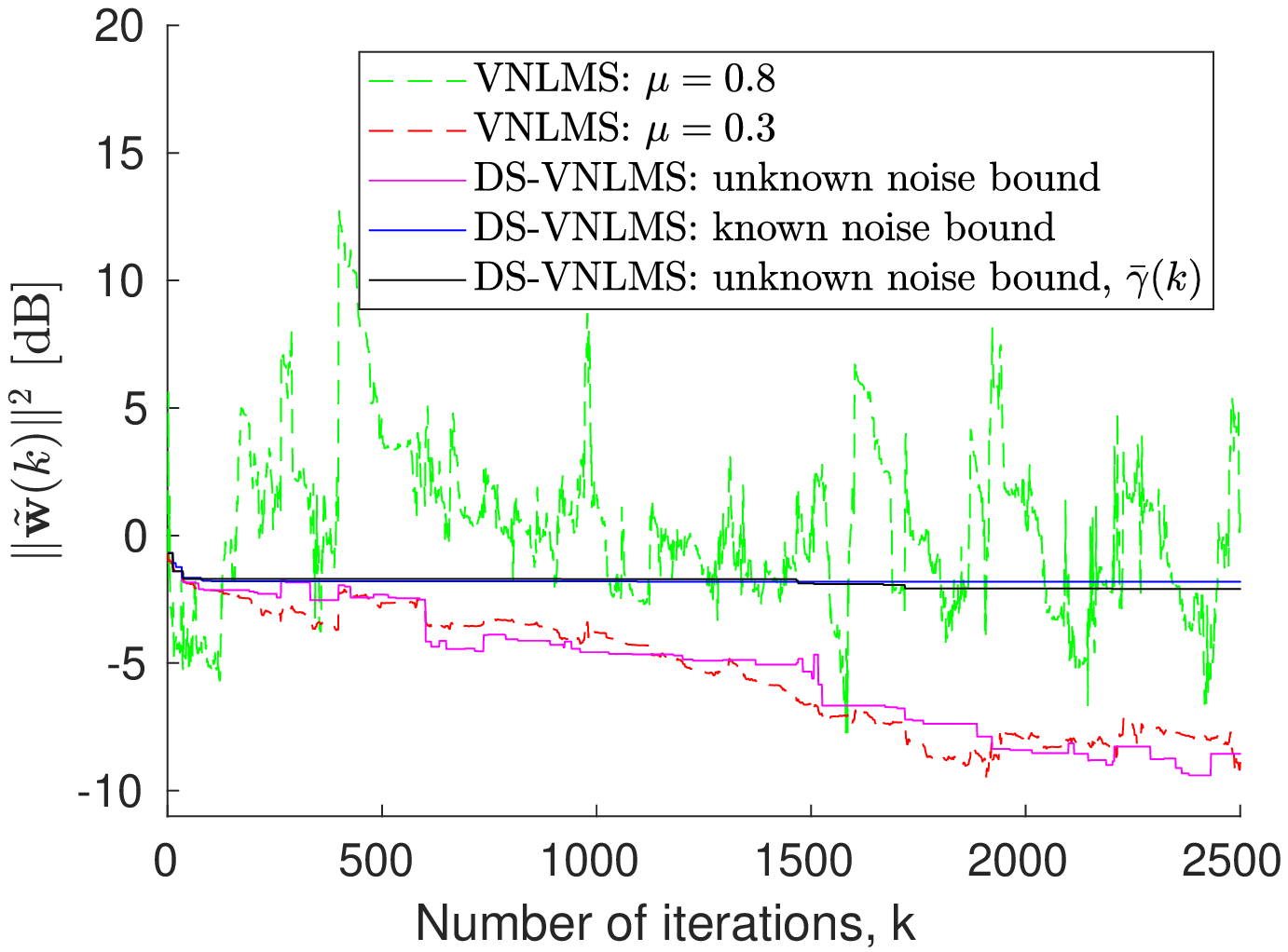}
		\label{fig:Volterra_all_correlated}}
	\caption{$\|\wbftilde(k)\|^2 \triangleq \| \wbf_* - \wbf(k) \|^2$ for the VNLMS and the DS-VNLMS algorithms: (a) the WGN input signal; (b) the AR(1) input signal. \label{fig:wtilde}}
\end{figure}


\section{Conclusions} \label{sec:conclusions}

In this paper, the robustness (in the sense of $l_2$-stability) of the DS-VNLMS algorithm has been analyzed.
First, we have reviewed the Volterra series and the robustness criterion, then the local robustness of the DS-VNLMS algorithm has been discussed.
Moreover, with the help of the local robustness summarized in Theorem~\ref{thm:local_robustness}, the global robustness property of the DS-VNLMS algorithm has been presented.
In other words, we have demonstrated that, when the energy of the additive noise signal is bounded, the DS-VNLMS algorithm never diverges, no matter how its parameters are adopted, and the energy of the errors is less than the energy of the uncertainties.
Furthermore, when the noise bound is known, we described how to choose a suitable $\gammabar$ such that the DS-VNLMS algorithm never produces a worse estimate.  
Also, for the situation in which the noise bound is unknown, we proposed a time-varying $\gammabar(k)$ that obtains high convergence rate and takes an efficient advantage of the input data. 
Finally, the numerical results substantiate the validity of the implemented analysis.


%



%
%

\ifCLASSOPTIONcaptionsoff
  \newpage
\fi

\bibliographystyle{alpha}



\end{document}